%% file: main.tex
\documentclass[runningheads]{llncs}

\usepackage[mobile]{eccv}

\usepackage{eccvabbrv}

\usepackage{graphicx}
\usepackage{booktabs}

\usepackage{amsmath}

\DeclareMathOperator*{\argmin}{arg\,min}

\usepackage[accsupp]{axessibility} 

\usepackage{hyperref}

\usepackage{orcidlink}

\usepackage{graphicx}
\usepackage{booktabs}
\usepackage{color}
\usepackage{longtable}

\usepackage{tabularx}
\usepackage{etoolbox,siunitx}

\usepackage[absolute]{textpos}

\usepackage{multibib}
\newcites{latex}{Supplementary-Literature}

\begin{document}

\title{Power Variable Projection \\ for Initialization-Free Large-Scale \\ Bundle Adjustment} 

\definecolor{somegray}{gray}{0.5}
\newcommand{\darkgrayed}[1]{\textcolor{somegray}{#1}}
\begin{textblock}{11}(2.5, -0.1)  %
\begin{center}
\darkgrayed{This paper has been accepted for publication at the European Conference on Computer Vision (ECCV), 2024. \copyright Springer}
\end{center}
\end{textblock}

\titlerunning{Power Variable Projection}

\author{Simon Weber\textsuperscript{1,2} \hspace{2em}
Je Hyeong Hong\textsuperscript{3} \hspace{2em}
Daniel Cremers\textsuperscript{1,2} \\
}

\institute{Technical University of Munich \and
Munich Center for Machine Learning
\and
Department of Electronic Engineering, Hanyang University
}

\authorrunning{S.Weber et al.}

\maketitle

\begin{abstract}

Most Bundle Adjustment (BA) solvers like the Levenberg-Marquardt algorithm require a good initialization.  Instead, initialization-free BA remains a largely uncharted territory. The under-explored Variable Projection algorithm (VarPro) exhibits a wide convergence basin even without initialization. Coupled with object space error formulation, recent works have shown its ability to solve small-scale initialization-free bundle adjustment problem. To make such initialization-free BA approaches scalable, we introduce Power Variable Projection (PoVar), extending a recent inverse expansion method based on power series. Importantly, we link the power series expansion to Riemannian manifold optimization. This projective framework is crucial to solve large-scale bundle adjustment problems without initialization. Using the real-world BAL dataset, we experimentally demonstrate that our solver achieves state-of-the-art results in terms of speed and accuracy. To our knowledge, this work is the first to address the scalability of BA without initialization opening new venues for initialization-free structure-from-motion.

  \keywords{Bundle Adjustment \and Initialization-Free \and Schur Complement \and Riemannian Manifold Optimization}
\end{abstract}

\section{Introduction}
\label{sec:intro}

Bundle adjustment (BA) is the key component of many structure-from-motion and 3D reconstruction algorithms. With the recent emergence of large-scale internet photo collections~\cite{agarwal2010bundle} and new applications (mixed reality, autonomous driving, digital twins), the need to solve large-scale BA has become an important challenge. Traditional BA addresses the following question: \textit{Given image measurements and approximate landmark positions and camera parameters, can we derive the exact positions and parameters?} The gold standard is to use the Levenberg-Marquardt algorithm~\cite{wright2006numerical} coupled with the Schur complement trick and a scalable solver for the reduced camera system, which is often the preconditioned conjugate gradient algorithm. Recent work~\cite{weber2023power} achieves outstanding speed for large-scale BA by using a power series expansion of the inverse Schur complement.

Recently, a new line of works \cite{hong2016projective,iglesias2023expose,hong2018pose} has attempted to solve the BA problem \textit{without} careful initialization: \textit{Given only image measurements, how do we derive pose parameters and 3D landmark positions?} This challenge is largely uncharted, and its scalability a blind spot. In particular, most existing works aim to formalize the problem into a stratified BA formulation, and none of them try to design effective solvers. It is noteworthy that even the most recent works only use direct factorization which becomes impractical for large-scale problems with several hundreds of cameras. In contrast to the traditional BA problem, the deficiency of competitive solvers for initialization-free BA can be broadly explained by the difference of convergence behaviour between a well-initialized problem and an initialization-free problem.

Following up on the recent findings concerning inverse expansion methods, we address the scalability of initialization-free BA. Our new solver based on the Variable Projection algorithm overcomes the issues of convergence of the scalable preconditioned conjugate gradients algorithm, while being efficient for thousand of camera viewpoints.
In summary, we make the following contributions: 
\begin{itemize}
\setlength\itemsep{0.8em}
   \item[$\bullet$] We introduce Power Variable Projection (\textit{PoVar}) for efficient large-scale bundle adjustment \textit{without} good initialization of camera poses and 3D landmarks. To the best of our knowledge, we are the first to address the scalability of initialization-free bundle adjustment formulation.
    \item[$\bullet$] We provide theoretical proofs that justify the extension of recent \textit{inverse expansion} method to the variable projection algorithm. While sharing a close algorithmic structure, the proposed extension and the existing \textit{power-series-based} method largely differ in the theory, in the applications and in the convergence behaviour.
    \item[$\bullet$] We theoretically extend the power series expansion for bundle adjustment to Riemannian manifold optimization. We take advantage of the matrix-specific structure to propose an efficient storage and memory-efficient computation for such optimization.
    \item[$\bullet$] We perform extensive evaluation of the proposed approach on the real-world BAL dataset. We emphasize the benefits of \textit{PoVar} in terms of scalability, speed and accuracy. In contrast to state-of-the-art solvers, our work is the first that solves large-scale bundle adjustment without initialization.
    \item[$\bullet$] We release our solver as open source to facilitate further research: \url{https://github.com/tum-vision/povar}.
\end{itemize}

\begin{figure}[tb]
\begin{center}
\includegraphics[width=1\textwidth]{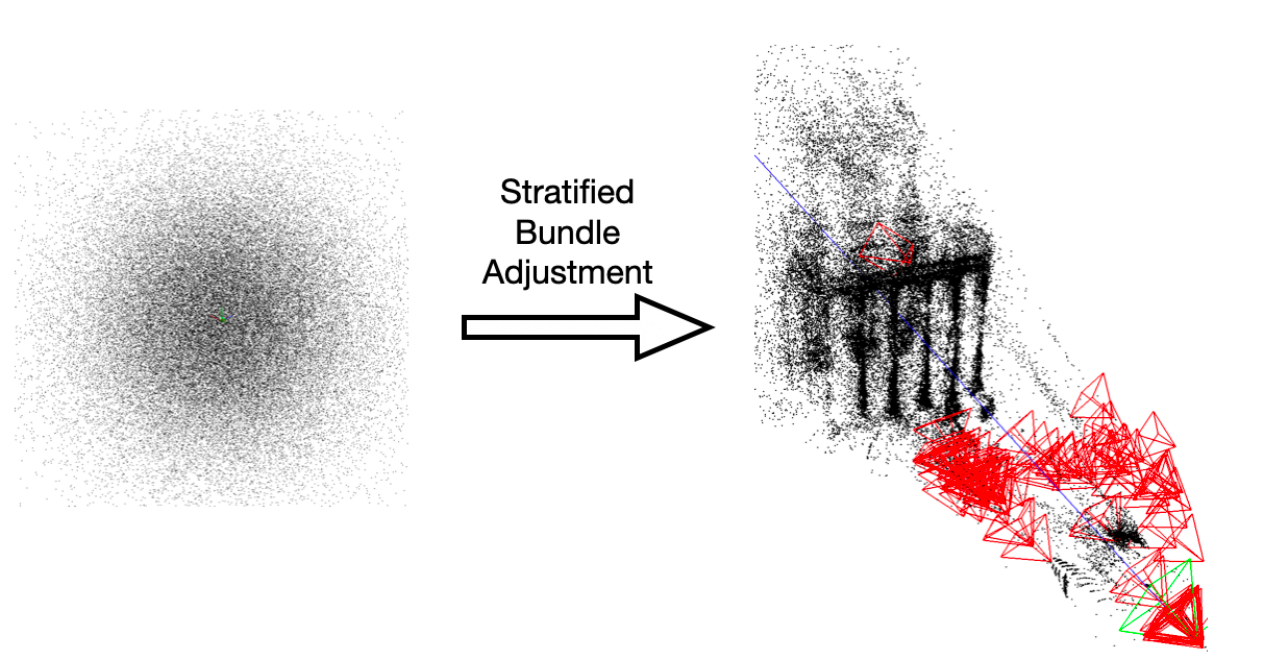}
\caption{In contrast to traditional bundle adjustment problem, initialization-free BA is a largely under-explored problem. It does not assume any approximation of pose and landmark parameters, making the problem much harder to solve. From a random initialization (\textit{left figure}), and given only image measurements, we aim to recover pose and landmark parameters. Our approach, that extends inverse expansion method, is motivated by the lack of scalability of existing solvers. On the real-world BAL problems (e.g. \textit{Venice-89}, \textit{right figure}), we demonstrate the efficiency of the proposed combination of our novel solver Power Variable Projection (PoVar) and Riemannian manifold optimization framework for expansion method to solve the stratified BA problem.}
\label{fig:intro}
\end{center}
\end{figure}

\section{Related Work}

As we address the scalability of the variable projection (VarPro) algorithm for initialization-free bundle adjustment (BA), we review works on VarPro and on BA from arbitrary initialization. We also provide some background on the inverse expansion methods. A more general description of BA can be found in~\cite{triggs2000bundle}.

\subsubsection*{Variable projection (VarPro) algorithm.}
VarPro is an optimization approach for solving bivariate problems that can be formulated as minimizing a cost function $f(u,v)$ with $f:\mathbb{R}^m\times\mathbb{R}^n \rightarrow \mathbb{R}$ over two sets of variables $u\in\mathbb{R}^m$ and $v\in\mathbb{R}^n$.
Unlike alternation, which fixes $u$ and optimizes over $v$ and vice versa, or joint optimization, which optimizes the stack of $u$ and $v$ simultaneously, variable projection replaces $v$ with $v^*(u):=\argmin_v f(u,v)$ ($v$ always optimal over $u$) such that the optimized cost function $f(u,v^*(u))=:f^*(u)$ becomes a function of $u$ only.
The original VarPro algorithm by Golub and Pereyra~\cite{golub1973differentiation} and its approximations by Ruhe and Wedin~\cite{ruhe1980separable} assume solving a separable nonlinear least squares (SNLS) problem where $v^*(u)$ can be obtained in closed form.
VarPro was consistently ignored in the computer vision community and even misidentified as a form of alternating optimization~\cite{buchanan2005dampednewton}.
It first receives proper attention when Okatani et al.~\cite{okatani2011dampedwiberg} demonstrates VarPro equipped with a trust-region approach such as Levenberg-Marquardt can yield a wide basin of convergence for several toy problems that can be formulated as a SNLS problem such as affine structure-from-motion and factorization-based non-rigid structure-from-motion.
Shortly after, Strelow~\cite{strelow2012l1wiberg,strelow2012l2wiberg} extends VarPro to the nonlinear case where $v^*(u)$ is not in closed form such as bundle adjustment.
Nevertheless, these works do not address the issue of increased algorithmic complexity.
Later, it has been shown by Hong et al.~\cite{hong2017revisiting} that VarPro can be efficiently implemented by performing inner iterations~\cite{ceres-solver} (also known as embedded point iterations~\cite{jeong2010embedded}) over the set of eliminated variables $u$ followed by performing a joint optimization step with no damping on $u$.
While this enables faster runtime compared to previous studies, it has only been tested up to small-medium sized problems with around 300 camera views.
To this date, no work to the best of our knowledge has improved the scalability of the VarPro algorithm beyond~\cite{hong2017revisiting}.

\subsubsection*{Initialization-free bundle adjustment.}
While traditional (large-scale) bundle adjustment is regularly studied (\cite{agarwal2010bundle,demmel2020distributed,weber2021multidirectional,demmel2021square,ren2022megba,zhou2020stochastic,belder2023game}), initialization-free BA is a recent research topic. In a seminal work, Hong et al. \cite{hong2016projective} propose to solve projective bundle adjustment from arbitrary initialization with the Variable Projection algorithm. Notably, they propose a stratified bundle adjustment formulation with increasing difficulty. In a follow-up work, pOSE \cite{hong2018pose} incorporates the equivalence between nonlinear VarPro and the Schur complement recently identified in~\cite{hong2017revisiting}. Additionally, they define an objective in-between affine and projective models that leads to a wide convergence basin. Iglesias et al.~\cite{iglesias2023expose} complements this \textit{pseudo object space error} formulation with an exponential regularization term, and show interesting results on very-small-scale problems. Nevertheless, to our knowledge, none of these studies look closely into the scalability of the proposed frameworks for initialization-free bundle adjustment.

\subsubsection*{Inverse expansion method.}
Despite its short recent appearance in the literature, inverse expansion method is a highly efficient and competitive approach for solving the linearized system of equations for bundle adjustment that is already challenging established factorization and iterative methods. It links the Schur complement \cite{zhang2006schur} to the power series expansion of its inverse. Using the \textit{power Schur complement} as a preconditioner for normal equations leads to good results for physics problems such as convection-diffusion \cite{zheng2021power}. PoBA \cite{weber2023power} proposes to directly apply the power Schur complement to the right-hand side of the reduced camera system. Their solver results in improved speed and accuracy for traditional bundle adjustment formulation with respect to the existing methods.

\section{Problem Statement and Motivation}
We consider a typical form of bundle adjustment. Given observations $m_{ij}$ for pose $i$ and landmark $j$, $K_{i}$, $R_{i}\in\,SO(3)$, $t_{i}\in\mathbb{R}^3$ the intrinsics, rotation and translation of pose $i$, and $x_{j}\in\mathbb{R}^3$,$\tilde{x}_{j}\in\,S^4$ the inhomogeneous and homogeneous landmark 3D positions (where $S^n$ denotes the set of all vectors on the unit $n$-sphere), we aim to solve:
\begin{equation}\label{eq:ba_traditional}
\min_{\{R_{i}\}\in SO(3),\{t_{i}\},\{\tilde{x}_j\}} \sum_{(i,j)\in \Omega}\lVert \pi (K_{i} [R_{i} | t_{i}]\tilde{x}_{j}) - m_{ij} \rVert^{2}_{2} \, ,
\end{equation}
where $\pi$ is the perspective projection $\pi([x,y,z]^{\top}) := [x/z, y/z]^{\top}$. 

Recently, Eq.~\ref{eq:ba_traditional} with good initialization has been solved efficiently with a novel inverse expansion method.
\subsection{Inverse expansion method}

Inverse expansion method \cite{zheng2021power,weber2023power} relates on the expansion of the inverse of a matrix into a power series, as stated in the following proposition:

\begin{proposition}\label{proposition}
Let $M$ be an $n \times n$ matrix. If the spectral radius of $M$ satisfies $\lVert M \rVert <1$, then 
\begin{equation} 
(I-M)^{-1} = \sum_{i = 0}^{m}M^{i} + R \, ,
\end{equation} 
where the error matrix 
\begin{equation} 
R = \sum_{i = m+1}^{\infty}M^{i} \, ,
\end{equation} 
satisfies 
\begin{equation}\label{bounded}
\lVert R \rVert \leq \frac{\lVert M \rVert ^{m+1}}{1 - \lVert M \rVert}  \, .
\end{equation}
\end{proposition}

Given a linear approximation of Eq.~\ref{eq:ba_traditional} followed by the Schur complement trick, Weber et al.~\cite{weber2023power} relates the inverse Schur complement of the Levenberg-Marquardt algorithm to its power series. They show significant improvement in terms of speed and accuracy to solve BA problem with good initialization. In addition to this empirical insights, some convergence behaviour concerning the approximated results for the BA problem are theoretically proved.

However, it is well-known in the literature (see e.g. \cite{hong2016projective}) that solving Eq.~\ref{eq:ba_traditional} from arbitrary initialization is non-feasible. Hong et al. \cite{hong2018pose} override this challenge by proposing a stratified bundle adjustment problem. Let us revisit the formulation of this approach.

\subsection{Initialization-free bundle adjustment}
The stratified BA problem is decomposed in two minimization problems, followed by a metric upgrade. Although we consider a pinhole camera model, the first stage assumes a projective model. 

\subsubsection{First stage: separable nonlinear optimization with projective camera.} 

Given $n_p$ poses and $n_l$ landmarks, $x = (x_p, x_l)$ contains all the optimization variables. We model the camera as a projective one. For pose $i$, we consider the camera parameters $x_{p}^{i} \in \mathbb{R}^{3 \times 4}$ and solve the following generic nonlinear separable problem:
\begin{equation}\label{eq:ba_first_stage}
\min_{x_{p},\tilde{x}_{l}} F(x_{p},\tilde{x}_{l}) = \big\lVert r(x_{p},\tilde{x}_{l}) \big\rVert_{2}^{2} =  \big\lVert G(x_{p})\tilde{x}_{l} - z(x_{p}) \big\rVert_{2}^{2} \, ,
\end{equation}
with $G(.)$ and $z(.)$ some linear operators, and the last coefficient of $\tilde{x}_{l}^{j}$ fixed to 1. For instance, pOSE \cite{hong2018pose} -- extensively used in our analysis (see Supplemental), proposes the following minimization problem: 
\begin{equation}\label{eq:pose}
F_{pOSE}(x_{p},x_{l}) = \sum_{(i,j)\in \Omega} \begin{Vmatrix} 
\sqrt{1-\eta}(x_{p}^{i, 1:2} \tilde{x}_{l}^{j} - (x_{p}^{i,3} \tilde{x}_{l}^{j})m_{ij}) \\
\sqrt{\eta}(x_{p}^{i,1:2}\tilde{x}_{l}^{j} - m_{ij}) 
\end{Vmatrix}_{2}^{2}\, ,
\end{equation}
with $x_{p}^{i,1:2}$ and $x_{p}^{i,3}$ respectively the first two rows and the third row of $x_{p}^{i}$, and $\eta \in [0,1]$.

Hong et al. \cite{hong2017revisiting} argue the superiority of VarPro over joint optimization to solve the previous equation, due to the random initialization of this stage.

\subsubsection*{Second stage: projective refinement.}
The cameras and landmarks parameters obtained by solving Eq.~\ref{eq:ba_first_stage} are refined by minimizing the projective standard objective \cite{hong2016projective} over the projective camera models in homogeneous form ($\{\tilde{x}_{p}^i ~|~ \mathrm{vec}(\tilde{x}_p^i) \in\, S^{12}\}$) and the 3D landmarks in homogeneous coordinates ($\{\tilde{x}_{l}^j~|~ \tilde{x}_l^j \in\, S^4 \})$:
\begin{equation}\label{eq:second_stage}
\sum_{(i,j) \in \Omega} \big\lVert \pi(\tilde{x}^{i}_{p}\tilde{x}^{j}_{l}) - m_{ij} \big\rVert_{2}^{2}.
\end{equation}
Importantly, the optimization is performed in homogeneous coordinates. Especially, Riemannian manifold optimization \cite{absil2008optimization} has to be incorporated. 

\subsubsection*{Metric upgrade.} The last stage is a minimization problem to enforce the projective camera matrices to satisfy $SE(3)$ properties. Note that our work mostly focuses on the first two stages. The proposed implementation for this third stage has illustrative purpose, see Fig.~\ref{fig:intro}. We refer the reader to Supplemental for further details.

\subsection{Limitations and proposed method}

In contrast to the Levenberg-Marquardt algorithm, few solvers have been designed to efficiently solve VarPro. In practice, even the most recent works \cite{hong2018pose,iglesias2023expose} use a direct factorization (e.g. Cholesky decomposition, QR factorization), that is well-known to be poorly scalable (see e.g. \cite{agarwal2010bundle}), to solve Eq.~\ref{eq:ba_first_stage}. In particular, we note that these works consider bundle adjustment problems with only few tens of cameras -- sometimes less than ten. On the other hand, Hong and Fitzgibbon \cite{hong2015secrets} show that coupling VarPro with the popular preconditioned conjugate gradients algorithm may not converge efficiently. 

We propose to build on recent inverse expansion method. We first adapt the power series expansion to the VarPro algorithm (\Cref{sec:povar}). We show that the so-called \textit{PoVar} efficiently solves Eq.~\ref{eq:ba_first_stage}. Moving forward, we extend the power series expansion to Riemannian manifold optimization (\Cref{sec:po_rie}). This new Riemannian framework, that we call \textit{RiPoBA}, is necessary to use expansion method for solving Eq.~\ref{eq:second_stage}. We demonstrate that the combination of this two solvers is highly competitive (\Cref{sec:experiments}).

\section{Power Variable Projection}

We start by revisiting the VarPro algorithm. We refer to \cite{golub1973differentiation} for further details.

\subsection{Revisited variable projection}
In contrast to joint optimization, VarPro optimizes over landmark parameters and camera parameters, but in a way different from the standard alternating least squares such that landmarks are not assumed to be fixed when updating the camera parameters.
It first considers Eq.~\ref{eq:ba_first_stage} as a nonlinear minimization problem over $\tilde{x}_{l}$ only. Due to the separability of the equation, a closed-form solution for optimal $\tilde{x}_{l}^{*}(x_{p})$ is straightforward:

\begin{equation}\label{eq:update_landmark}
\tilde{x}_{l}^{*}(x_{p}) = \argmin_{\tilde{x}_{l}} \lVert G(x_{p})\tilde{x}_{l} - z(x_{p}) \rVert_{2}^{2} = G(x_{p})^{\dag}z(x_{p}) \, ,
\end{equation}
with $G(x_{p})^{\dag}$ the pseudo-inverse of $G(x_{p})$. As we set the last coordinate of $\tilde{x}_{l}$ to $1$, we normalize $\tilde{x}_{l}^{*}(x_{p})$ by its last coordinate and its substitution in Eq.~\ref{eq:ba_first_stage}  leads to the following reduced problem:
\begin{equation}\label{eq:reduced_varpro}
\min_{x_{p}} r^{*}(x_{p}) = \min_{x_{p}} \lVert (G(x_{p})G^{\dag}(x_{p})-I)z(x_{p}) \rVert_{2}^{2} \,.
\end{equation}
Following \cite{hong2017revisiting}, Eq.~\ref{eq:reduced_varpro} can be solved with LM algorithm over $x_{p}$. The Jacobian of $r^{*}$ is approximated with the so-called RW2 approximation \cite{kaufman1975variable}, that leads to the normal equation: 
\begin{equation}
    \begin{pmatrix}
    U_{\lambda} && W \\
    W^{\top}    && V_{0}
\end{pmatrix} \begin{pmatrix}
\Delta x_{p} \\ \Delta x_{l}
\end{pmatrix} = - 
\begin{pmatrix}
b_{p} \\ b_{l}
\end{pmatrix} \, ,
\end{equation}
where 
\begin{align}
U_{\lambda} = J_{p}^{\top}J_{p} + \lambda D_{p}^{\top}D_{p} \, , \\
V_{0} = J_{l}^{\top} J_{l} \, , \quad W = J_{p}^{\top}J_{l} \, , \\
b_{p} = J_{p}^{\top}r^{0} \, , \quad b_{l} = J_{l}^{\top}r^{0} \, ,
\end{align}
with $J_{l}$ and $J_{p}$ respectively the landmark and pose Jacobians of the original residual (Eq.~\ref{eq:ba_first_stage}) around an equilibrium $r^{0}$, and $D_{p}$ diagonal damping matrix for pose variables.
In contrast to joint optimization, only the pose Jacobian is damped in the associated Hessian. It follows that $U_{\lambda}$ is symmetric positive-definite \cite{triggs2000bundle}, whereas $V_{0}$ is only guaranteed to be symmetric positive-semidefinite.
Nevertheless, we observe $V_{0}$ is usually of full rank unless a 3D landmark is observed by few cameras with narrow baselines.
By using the Schur complement trick, the update equation for VarPro becomes:
\begin{equation}\label{eq:pose_update}
(U_{\lambda} - WV_{0}^{-1}W^{\top})\Delta x_{p} = b_{p}-WV_{0}^{-1}b_{l}\, .
\end{equation} 
Note that the Schur complement associated to VarPro:
\begin{equation}\label{eq:schur_varpro}
S^{V} = U_{\lambda} - WV_{0}^{-1}W^{\top} \, ,
\end{equation}
while sharing a close structure to the Schur complement of the traditional BA problem, has a  different convergence behaviour, due to the undamped landmark Jacobian $V_{0}$.

\subsection{Power series for VarPro}\label{sec:povar}
Inspired by Weber et al. \cite{weber2023power}, the following lemma holds, even if $V_{0}$ is only symmetric positive-semidefinite:
\begin{lemma} 
Let $\mu$ be an eigenvalue of $U_{\lambda}^{-1}WV_{0}^{\dag}W^{\top}$. Then 
\begin{equation}
    0 \leq \mu < 1 \, .
\end{equation}
\end{lemma}
\begin{proof}
    We refer the reader to the Supplemental.
\end{proof}
It follows that the pose update $\Delta x_{p}$ in Eq.~\ref{eq:pose_update} can be directly approximated\footnote{As in practice $V_{0}$ is full-rank, and to avoid encumbering notations, we write $V_{0}^{-1}$ instead of the pseudo-inverse in the rest of the paper.} by $x(m)$ with \Cref{proposition} applied to the Schur complement $S^{V}$:
\begin{equation}
x(m) = - \sum_{i=0}^{m}(U_{\lambda}^{-1}WV_{0}^{-1}W^{\top})^{i}U_{\lambda}^{-1} (b_{p}-WV_{0}^{-1}b_{l}) \, .
\end{equation}
Once the pose update is estimated, the landmark update is derived following closed-form Eq.~\ref{eq:update_landmark}.
That extends the inverse expansion method to Variable Projection algorithm. The difference between both solvers is the role of the damped parameters. While it seems slight, this is enough to offer a dissimilar convergence behaviour, as we will see in the experiments. 

Before that, let us investigate Riemannian manifold optimization, that is necessary to solve the second stage of the stratified BA problem. In particular, we show that we can link this framework to expansion method.

\subsection{Power Riemannian manifold optimization}\label{sec:po_rie}



As the projective refinement step in Eq.~\ref{eq:second_stage} involves both camera matrices and 3D landmarks in homogeneous forms, we exhibit local scale freedom for both camera and landmark parameters.
This necessitates incorporation of the Riemannian manifold optimization framework without which the linearized system of equations is always rank-deficient and unsolvable. 
While a complete theoretical overview of such framework can be found in \cite{absil2008optimization}, we formalize Eq.~\ref{eq:second_stage} as:
\begin{equation}
\argmin_{\Delta \tilde x_{p}, \Delta \tilde x_{l}} \lVert f(\tilde x_{p}+\Delta \tilde x_{p}, \tilde x_{l} + \Delta \tilde x_{l}) \rVert_{2}^{2},
\end{equation}
where $\tilde x_p\in\mathbb{R}^{12 n_{p}}$ denotes the stack of vectorized homogeneous camera parameters, $\tilde x_l\in\mathbb{R}^{4 n_{l}}$ denotes the stack of homogeneous 3D landmarks and $\Delta \tilde x_{p}\in\mathbb{R}^{12 n_{p}}$ and $\Delta \tilde x_l\in\mathbb{R}^{4 n_{l}}$ are the updates in homogeneous camera parameters and 3D landmarks respectively.
The unknowns are searched in the tangent space of current $\tilde x = [\tilde x_{p}^\top, \tilde x_{l}^\top]^\top$, that we note $\tilde x^{\perp}\in\mathbb{R}^{(12n_{p}+4n_{l}) \times (11n_{p}+3n_{l})}$ such that $(\tilde x^{\perp})^{\top} x = 0$. To simplify notations, $\tilde x_{p}$ is considered as a vector in this section, and $\tilde x_{p}^{\perp}\in\mathbb{R}^{12n_{p} \times 11n_{p}}$, $\tilde x_{l}^{\perp}\in\mathbb{R}^{4n_{l} \times 3n_{l}}$ are block-diagonal, each block corresponding to the associated pose $i$ and landmark $j$, respectively. The projective refinement becomes:
\begin{equation}
\argmin_{\Delta \tilde x \perp \tilde x} \lVert f(\tilde x_{p}+\Delta \tilde x_{p}, \tilde x_{l} + \Delta \tilde x_{l}) \rVert_{2}^{2} \, .
\end{equation}
By coupling Riemannian manifold optimization and LM algorithm, and according to \cite{absil2008optimization}, we get the following normal equation, projected onto the tangent space of $\tilde x$:
\begin{equation}\label{eq:riemannian_normal}
\begin{pmatrix}
    (\tilde x_{p}^{\perp})^{\top}U_{\lambda}\tilde x_{p}^{\perp} && (\tilde x_{p}^{\perp})^{\top}W \tilde x_{l}^{\perp} \\
    (\tilde x_{l}^{\perp})^{\top}W^{\top}\tilde x_{p}^{\perp}    && (\tilde x_{l}^{\perp})^{\top}V_{\lambda}\tilde x_{l}^{\perp}
\end{pmatrix} \begin{pmatrix}
\Delta x_{p} \\ \Delta x_{l}
\end{pmatrix} = - 
\begin{pmatrix}
(\tilde x_{p}^{\perp})^{\top}b_{p} \\ (\tilde x_{l}^{\perp})^{\top} b_{l} 
\end{pmatrix}
\, .
\end{equation}
where $\Delta x_p\in\mathbb{R}^{11n_{p}}$ and $\Delta x_l\in\mathbb{R}^{3n_{l}}$ are the camera update and the landmark update respectively made on the tangent space of $x$.
By keeping coherent notations we can note the projected Jacobians and the projected damping parameters onto the tangent space of $x$ as: 
\begin{align}
\tilde{J}_{p} = J_{p} \tilde x_{p}^{\perp} \, , \quad  \tilde{J}_{l} = J_{l} \tilde x_{l}^{\perp} \, , \quad \tilde{\lambda}_{p} = (\tilde x_{p}^{\perp})^{\top}\lambda \tilde x_{p}^{\perp} \, , \quad \tilde{\lambda}_{l} = (\tilde x_{l}^{\perp})^{\top}\lambda \tilde x_{l}^{\perp} \, ,
\end{align}
and then Eq.~\ref{eq:riemannian_normal} becomes: 
\begin{equation}
\begin{pmatrix}
    \tilde{U}_{\tilde{\lambda}} && \tilde{W} \\
    \tilde{W}^{\top}    && \tilde{V}_{\tilde{\lambda}}
\end{pmatrix} \begin{pmatrix}
\Delta x_{p} \\ \Delta x_{l}
\end{pmatrix} = - 
\begin{pmatrix}
\tilde{b}_{p} \\ \tilde{b}_{l} 
\end{pmatrix} \, ,
\end{equation}
where 
\begin{align}
\tilde{U}_{\tilde{\lambda}} = \tilde{J}_{p}^{\top}\tilde{J}_{p} + D_{p}^{\top} \tilde{\lambda}_{p} D_{p} \, , \\
\tilde{V}_{\tilde{\lambda}} = \tilde{J}_{l}^{\top} \tilde{J}_{l} + D_{l}^{\top} \tilde{\lambda}_{l} D_{l} \, , \\ W = \tilde{J}_{p}^{\top}\tilde{J}_{l} \, , \\
\tilde{b}_{p} = \tilde{J}_{p}^{\top}r^{0} \, , \quad \tilde{b}_{l} = \tilde{J}_{l}^{\top}r^{0} \, ,
\end{align}

We have unified the notations of bundle adjustment with Riemannian manifold optimization. As the projection $x^{\perp}$ is full-rank, it follows that $\tilde{U}_{\tilde{\lambda}}$ and $\tilde{V}_{\tilde{\lambda}}$ are symmetric positive-definite (see Supplemental), and then the associated Riemannian Schur complement:
\begin{equation}
\tilde{S} = \tilde{U}_{\tilde{\lambda}} - \tilde{W}\tilde{V}_{\tilde{\lambda}}^{-1}\tilde{W}^{\top} 
\end{equation}
satisfies the assumption of \Cref{proposition}:
\begin{lemma}
Let $\tilde{\mu}$ be an eigenvalue of $\tilde{U}_{\tilde{\lambda}}^{-1}\tilde{W}\tilde{V}_{\tilde{\lambda}}^{-1}\tilde{W}^{\top}$. Then 
\begin{equation}
    0 \leq \tilde{\mu} < 1 \, .
\end{equation}
\end{lemma}
The power series expansion can be applied to the inverse Riemannian Schur complement:
\begin{equation}
\tilde{S}^{-1} \approx \sum_{i=0}^{m}(\tilde{U}_{\tilde{\lambda}}^{-1}\tilde{W}\tilde{V}_{\tilde{\lambda}}^{-1}\tilde{W}^{\top})^{i}\tilde{U}_{\tilde{\lambda}}^{-1} \, ,
\end{equation}
to get pose updates, and then landmark updates by back-substitution.

Finally, the homogeneous pose and landmark updates   $\Delta \tilde x_{p}$ and $\Delta \tilde x_{l}$ are retrieved by back-projecting pose and landmark updates in the tangent space to the original vector space dimension as follows:
\begin{align*}
\Delta \tilde{x}_{p} = \tilde{x}_{p}^{\perp}\Delta x_{p} \, , \quad \Delta \tilde{x}_{l} = \tilde{x}_{l}^{\perp}\Delta x_{l}.
\end{align*}
After above updates are added to the pose and landmark parameters $x$, we carry out manifold retraction by normalizing individual vectors of camera parameters and 3D landmarks to maintain their normalized homogeneous forms.

We call this projective framework \textit{Riemannian PoBA} (RiPoBA), that extends PoBA \cite{weber2023power} to the Riemannian manifold optimization framework.

\begin{figure}[tb]
\begin{center}
\includegraphics[width=1\textwidth]{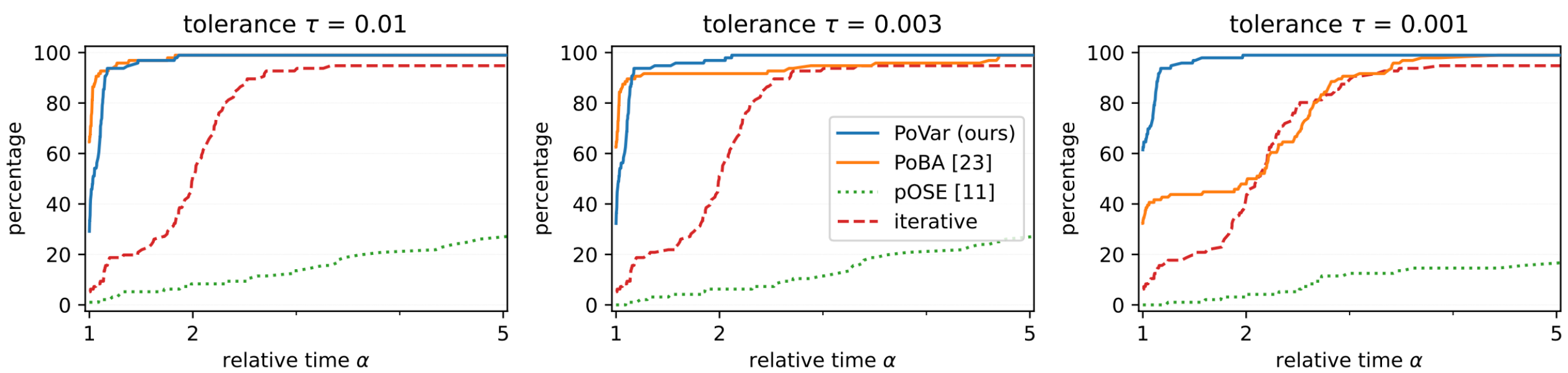}
\caption{Average performance profiles across all BAL problems for solving the first stage Eq.~\ref{eq:ba_first_stage}. Given a tolerance $\tau \in \{0.01, 0.003, 0.001 \}$, it represents the percentage of solved problems ($y$-axis) with relative runtime $\alpha$ ($x$-axis). Expansion methods \textit{PoVar} and \textit{PoBA} show outstanding speed-accuracy results. Our solver \textit{PoVar} is competitive, and most notably for the highest accuracy $\tau = 0.001$.}
\label{fig:stage1_pp}
\end{center}
\end{figure}

\begin{figure}[tb]
\begin{center}
\includegraphics[width=1\textwidth]{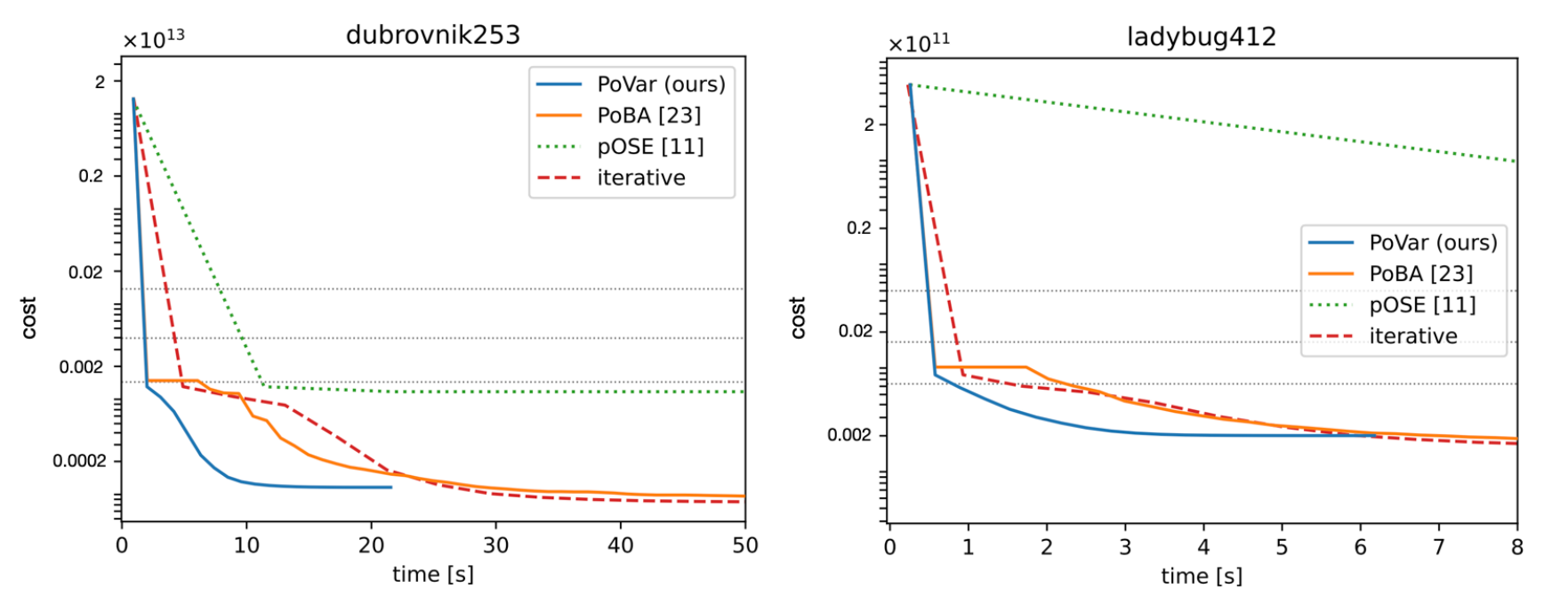}
\caption{Convergence plots of \textit{Dubrovnik-253} (left) from BAL datasets with 253 poses and \textit{Ladybug-412} with 412 poses, for solving the first stage (Eq.~\ref{eq:ba_first_stage}). The dotted lines correspond to cost thresholds for tolerance $\tau \in \{0.01, 0.003, 0.001 \}$.}
\label{fig:stage1_cost}
\end{center}
\end{figure}

\begin{figure}[tb]
\begin{center}
\includegraphics[width=1\textwidth]{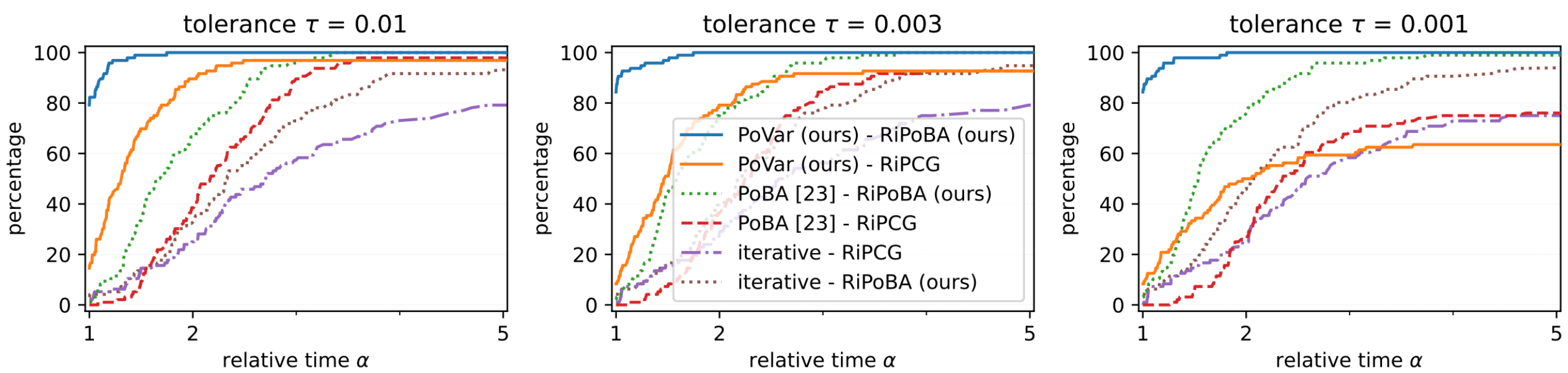}
\caption{Performance profiles for all BAL problems for solving the first two stages (Eq.~\ref{eq:ba_first_stage} and Eq.~\ref{eq:second_stage}). In each combination, the first solver is used to solve the first stage, and the second solver is used for the second stage. Our proposed combination \textit{PoVar} followed by a Riemannian expansion method outperforms by a large margin compared to other competitors. Also, our proposed Riemannian-expansion method \textit{RiPoBA} outperforms the iterative baseline in all cases, given a same solver for the first stage.}
\label{fig:stage2_pp}
\end{center}
\end{figure}

\begin{figure}[tb]
\begin{center}
\includegraphics[width=1\textwidth]{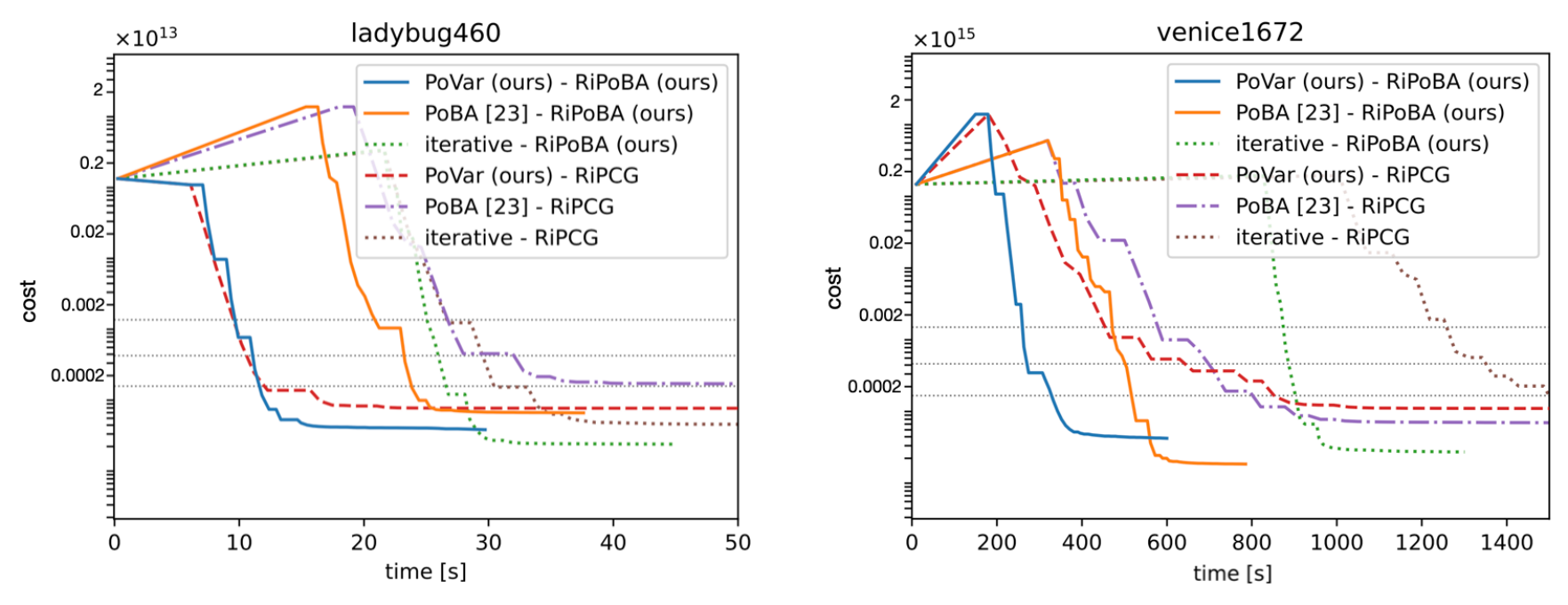}
\caption{Convergence plots of \textit{Ladybug-460} (left) from BAL datasets with 460 poses and \textit{Venice-1672} with 1672 poses, for solving the second stage (Eq.~\ref{eq:second_stage}). The dotted lines correspond to cost thresholds for tolerance $\tau \in \{0.01, 0.003, 0.001 \}$. Note that for fair comparison, the initial cost is derived before the first stage. The second cost -- that may be higher than the initial one, is the initial cost of the second stage, after the first stage has been run. The runtime includes the time spent to solve the first stage.}
\label{fig:stage2_cost}
\end{center}
\end{figure}

\section{Experiments}\label{sec:experiments}

\subsection{Implementation}
We implement \textit{pOSE}\footnote{We use our custom implementation as the official code of \cite{hong2018pose} is not publicly available.} \cite{hong2018pose}, \textit{PoVar} and \textit{RiPoBA} framework in C++, directly on the publicly available implementation of PoBA \footnote{\url{https://github.com/simonwebertum/poba}} \cite{weber2023power}. That leads to fair comparisons with this recent and challenging solver. \textit{pOSE} differs from \textit{PoVar} by the use of a direct sparse Cholesky factorization. We also compare to VarPro with the conjugate gradients algorithm, preconditioned with Schur-Jacobi preconditioner, called \textit{iterative} in our experiments. For the second stage, we compare \textit{RiPoBA} to the conjugate gradients algorithm preconditioned by Schur-Jacobi preconditioners with Riemannian manifold optimization framework, called \textit{RiPCG}. Except the solver itself, all implementations share much of the code with \cite{weber2023power}. We run experiments on MacOS 14.2.1 with an Intel Core i5 (4 cores at 2GHz).
\subsection{Experimental settings}

\subsubsection{Setup.}
For each stage, we set the maximum number of iterations to $50$, stopping earlier if a relative function tolerance of $10^{-6}$ is achieved. The damping factor $\lambda$ starts for each stage at $10^{-4}$ and is updated accordingly to the success or failure of the iteration. For expansion methods, we set the maximal order of power series to $20$ and a threshold to $0.01$. For iterative methods, we set the maximum number of inner iterations to $500$. We set $\eta$ for pOSE~(Eq.~\ref{eq:pose}) to $0.1$.

\begin{figure}[tb]
\begin{center}
\includegraphics[width=1\textwidth]{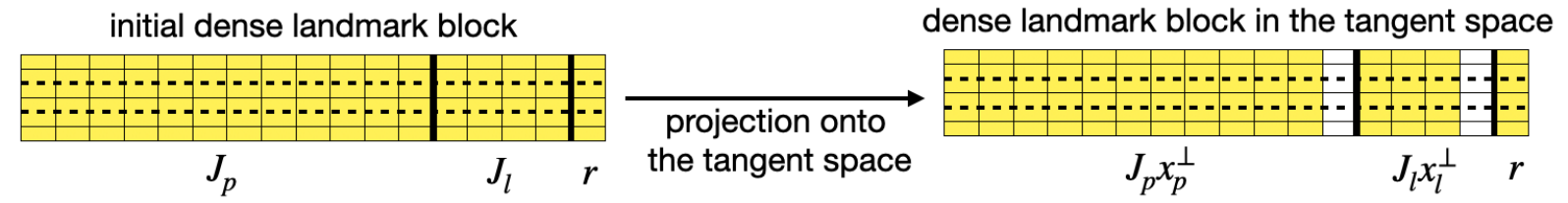}
\caption{Example of dense landmark block in the tangent space with 3 observations from a single landmark. We project the initial dense landmark block (\textit{left}) onto the tangent space by applying to each $i$-th pose Jacobian $J_{p_i} \in \mathbb{R}^{2 \times 12}$ the projection $\tilde x_{p_i}^{\perp}$, and to each $j$-th landmark Jacobian $J_{l_j} \in \mathbb{R}^{2 \times 4}$ the projection $\tilde x_{l_j}^{\perp}$. The resulting pose and landmark Jacobians in the tangent space (\textit{right}) belong to $\mathbb{R}^{2 \times 11}$ and $\mathbb{R}^{2 \times 3}$ respectively.}
\label{fig:dense_storage}
\end{center}
\end{figure}

\subsubsection{Efficient storage for Riemannian manifold optimization framework.} As in \cite{weber2023power}, we leverage the special structure of BA problem and propose a memory-efficient storage. We organize the landmarks into dense blocks. In particular, we apply on each row associated to a landmark the block-matrices of the projection $\tilde x^{\perp}_{l_j},\tilde x^{\perp}_{p_i}$ corresponding to the landmark and to the cameras stored in the considered dense landmark block (see Fig.~\ref{fig:dense_storage}). 
\subsubsection{Dataset.}
We extensively evaluate our solver and the baselines on the $97$ real-world bundle adjustment problems from the BAL project page \cite{agarwal2010bundle}. The number of poses goes from $16$ to $13682$.  We refer the reader to Supplemental for further details about these problems. For each problem, we only keep the observation measurements. Pose parameters are randomly drawn from an isotropic Gaussian distribuaton with mean $0$ and variance $1$, and landmark parameters are deduced with Eq.~\ref{eq:update_landmark}. Notably and contrary to previous works on initialization-free BA, each solver is ran on the same randomized problem, for fair comparisons.

\subsection{Performance profile}
We jointly evaluate both runtime and accuracy with performance profiles \cite{dolan2002benchmarking}. Given a solver, the performance profile maps the relative runtime $\alpha$ to the percentage of problems solved with accuracy $\tau$. Graphically, the performance profile of a given solver is the percentage of problems solved faster than the relative runtime $\alpha$ on the $x$-axis. Let be $S$ and $P$ the sets of solvers and problems, respectively. In practice, we can define the objective threshold for a problem $p$ by:
\begin{equation}\label{eq:pp}
f_{\tau}(p) = f^{*}(p) + \tau (f^{0}(p) - f^{*}(p)) \, ,
\end{equation}
with $f^{0}(p)$ the initial objective and $f^{*}(p)$ the smallest error reached by the family of solvers. The runtime a solver $s$ needs to reach this threshold is noted $T_{\tau}(p,s)$. The performance profile of a solver for a relative runtime $\alpha$ is defined as:
\begin{equation}
\rho (s,\alpha) = \frac{100}{\lvert P \rvert} \lvert \{ p \in P | T_{\tau}(p,s) \leq \alpha \min_{s\in S} T_{\tau}(p,s) \} \lvert \, .
\end{equation}
Graphically, a curve on the left of the performance profile is linked to better runtime, whereas a curve on the right is related to better accuracy. Note that for meaningful comparison, all solvers should have the same initial objective.

\subsection{Analysis}
\subsubsection{First stage.} 
Fig.~\ref{fig:stage1_pp} shows the performance profiles for all BAL datasets with tolerances $\tau \in \{0.01, 0.003, 0.001\}$ to solve Eq.~\ref{eq:ba_first_stage}. As expected, the direct factorization solver (dashed green) used in Hong et al. \cite{hong2018pose} shows poor performance. Our solver \textit{PoVar} (blue) challenges \textit{PoBA} for the largest tolerance $0.01$, and is by far the most competitive solver for the smallest tolerance $\tau = 0.001$, that is for the highest accuracy. For $\tau = 0.003$, \textit{PoVar} slightly outperforms \textit{PoBA}. We also highlight that our solver outperforms the two competitors associated to the VarPro algorithm, iterative (dashed red) and direct factorization. 

Fig.~\ref{fig:stage1_cost} shows two examples of the cost decrease during the first stage. Notably, \textit{PoVar} (blue) converges much more smoothly than its main challenger \textit{PoBA} which gets stuck in early iterations. By considering the intersections of the solvers with the cost thresholds (dashed grey lines), it also demonstrates the slowness of the iterative method compared to expansion methods in terms of runtime.
\subsubsection{First and second stage.}
Fig.~\ref{fig:stage2_pp} shows the performance profiles for all BAL datasets with tolerances $\tau \in \{ 0.01, 0.003, 0.001 \}$ to solve the first two stages, that are Eq.~\ref{eq:ba_first_stage} followed by Eq.~\ref{eq:second_stage}. Note that we compare in this experiment the cost of the second stage only, as the first stage is only used to get an approximated initialization for the projective formulation. As direct factorization shows poor performance during the first stage, we only take into account the most promising combinations of solvers. \textit{PoVar} followed by \textit{RiPoBA} (blue) outperforms all the competitors both in terms of runtime and accuracy. The combinations with \textit{RiPoBA} outperform \textit{RiPCG} for the highest accuracy $\tau = 0.001$ for all relative time greater than $\alpha = 2$, that reflects the better convergence of Riemannian expansion method compared to \textit{RiPCG}. We also note, given a same solver for first stage, \textit{RiPoBA} outperforms \textit{RiPCG} during the second stage across all tolerances.

Fig.~\ref{fig:stage2_cost} illustrates on two examples the cost decrease during the second stage. On the left figure (\textit{Ladybug-460}), the best two solvers in terms of final convergence are built on our framework \textit{RiPoBA}. On the right figure (\textit{Venice-1672}), all the solvers using \textit{RiPoBA} converge to a smaller error than their iterative competitor \textit{RiPCG}. By considering the intersection between the solvers and the cost thresholds (dashed grey lines), the combination \textit{PoVar}-\textit{RiPoBA} outperforms all other combinations, and most notably for the largest problem with $1672$ poses. 

\subsubsection{Conclusive remark.} 
The experiments emphasize the high efficacy of our solvers \textit{PoVar} and \textit{RiPoBA}, during both first and second stages of the stratified BA problem. Concerning the first stage, the convergence of \textit{PoVar} is much smoother than its competitors, that explains its larger speed with respect to \textit{PoBA} to reach the cost thresholds, albeit both are built with power series. Regarding the second stage, for a same given solver in the first stage, our \textit{RiPoBA} outperforms the preconditioned conjugate gradients with Riemannian manifold optimization framework in terms of speed and accuracy, especially when coupled with \textit{PoVar}.

\section{Conclusion}

We have introduced a novel approach to address the scalability challenge for initialization-free bundle adjustment. Our proposed Power Variable Projection (PoVar) algorithm, theoretically justified, offers new insights to this uncharted problem. By extending recent inverse expansion techniques to the VarPro algorithm on one hand, and to Riemannian manifold optimization on the other hand, we have demonstrated the capability to efficiently solve large-scale stratified BA problem with thousands of cameras. Notably, we achieve state-of-the-art results in terms of speed and accuracy on the real-world BAL dataset. While initialization-free BA is still in its nascent stage, we hope that our method will pave the way for further exploration of this difficult optimization problem, and will generate further steps towards initialization-free structure-from-motion. 

\subsubsection{Limitations and future work.}
First, our analysis applies to the BA problem in which excess outlier point tracks are assumed to have been filtered out as in previous formulations \cite{hong2016projective,hong2017revisiting,hong2018pose,iglesias2023expose}. Second, the 3D reconstruction in Fig.~\ref{fig:intro} assumes that the intrinsics are known during the metric upgrade stage. In practice, we use the approximated focal lengths given in the BAL dataset, that results in imperfectly accurate illustration. However, some formulations based on pOSE handle unknown intrinsics and could be easily adapted to \textit{PoVar}.

\subsection*{Acknowledgements}
This work was supported in part by the NRF grants funded by the Korea government (MSIT) (No. 2022R1C1C1004907) and in part by the Institute of Information \& communications Technology Planning \& Evaluation (IITP) under the artificial intelligence semiconductor support program to nurture the best talents (IITP-(2024)-RS-2023-00253914) grant funded by the Korea government(MSIT).

\bibliographystyle{splncs04}
\bibliography{main}
\clearpage

\title{Power Variable Projection \\ \vspace{-0.1em} for Initialization-Free Large-Scale \\ \vspace{-0.1em} Bundle Adjustment \\         \large \vspace{0.5em}-- Supplementary Material -- \\
        \vspace{2em}}

\titlerunning{Power Variable Projection - Appendix}
\author{}
\institute{}
\maketitle
\appendix

\noindent This supplemental material is organized as follows:\\
\noindent \textbf{\cref{sec:robust}} studies robustness of \textit{PoVar} with respect to $\eta$ and random initialization, as well as the scale of the considered problems.\\
\noindent \textbf{\cref{sec:lemma}} complements the theoretical justifications of \textit{PoVar} and \textit{RiPoBA}.\\
\noindent \textbf{\cref{sec:pose_formulation}} briefly comments a recent follow-up formulation of pOSE error.\\ 
\noindent \textbf{\cref{sec:metric}} addresses the metric upgrade stage, necessary to estimate the projective transformation and to get Euclidean reconstruction.\\
\noindent \textbf{\cref{sec:dataset}} gives more details about the $97$ BAL problems used in our experiments.\\

\begin{figure}[tb]
\begin{center}
\includegraphics[width=1\textwidth]{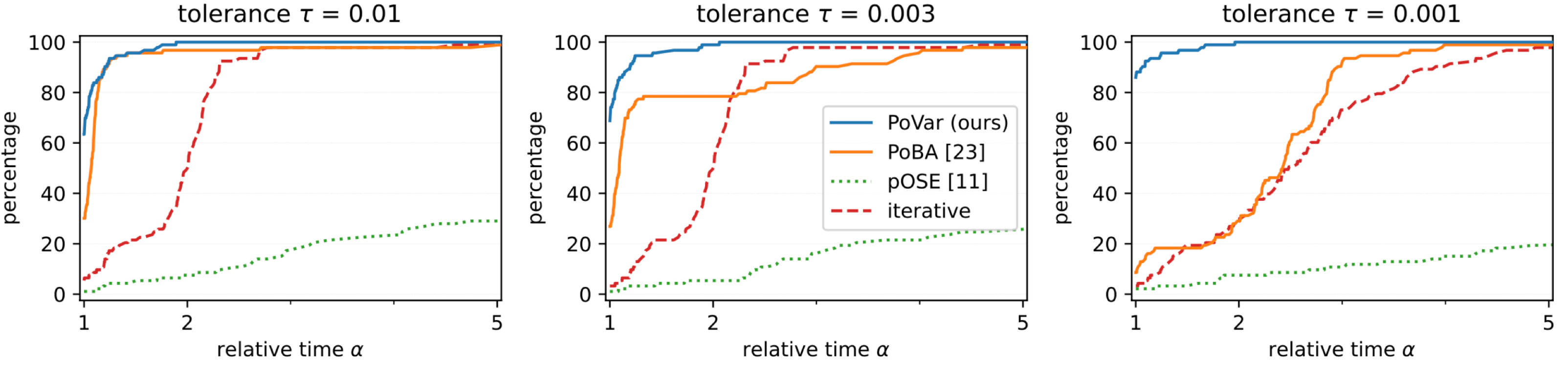}
\caption{With $\eta = 0.2$, performance profiles for all real-world BAL problems for solving the first stage (6). Given a tolerance $\tau \in \{0.01, 0.003, 0.001 \}$, it represents the percentage of solved problems ($y$-axis) with relative runtime $\alpha$ ($x$-axis). Our solver \textit{PoVar} is very competitive, and most notably for the highest accuracy $\tau = 0.001$ and $\tau=0.003$.}
\label{fig:sup_eta20}
\end{center}
\end{figure}

\begin{figure}[tb]
\begin{center}
\includegraphics[width=1\textwidth]{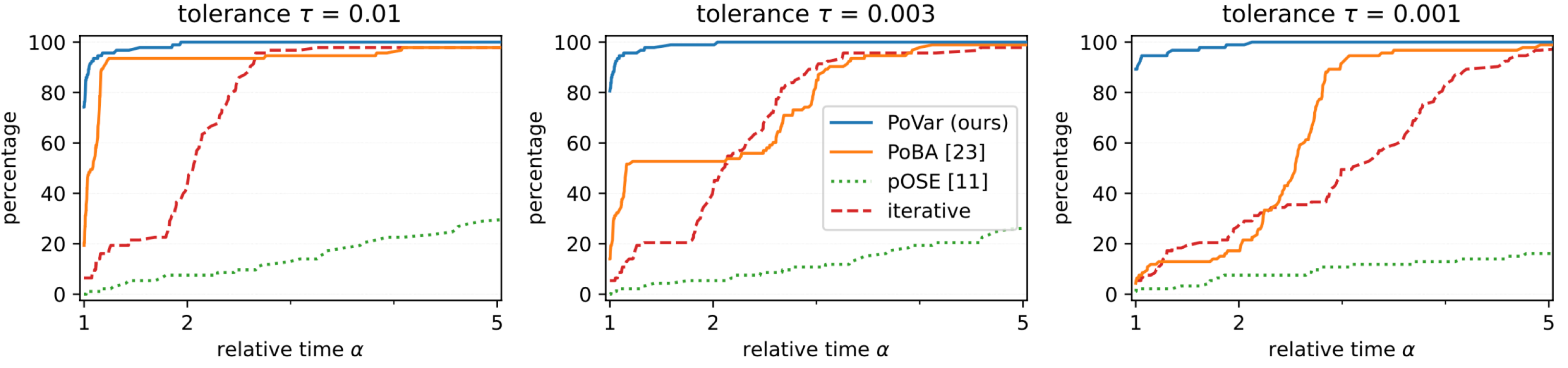}
\caption{With $\eta = 0.3$, performance profiles for all real-world BAL problems for solving the first stage (6).}
\label{fig:sup_eta30}
\end{center}
\end{figure}

\begin{figure}[tb]
\begin{center}
\includegraphics[width=1\textwidth]{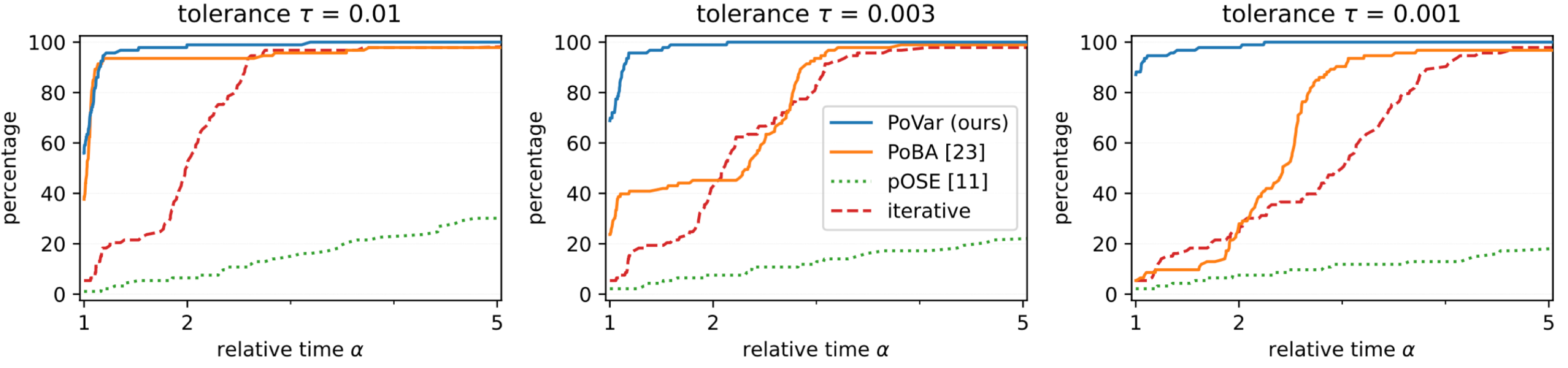}
\caption{With $\eta = 0.4$, performance profiles for all real-world BAL problems for solving the first stage (6).}
\label{fig:sup_eta40}
\end{center}
\end{figure}

\section{Robustness}\label{sec:robust}

We illustrate the robustness of our solver \textit{PoVar} for solving the first stage with respect to the coefficient $\eta$ in the pOSE formulation. \Cref{fig:sup_eta20}, \Cref{fig:sup_eta30} and \Cref{fig:sup_eta40} represent the performance profile for $\eta = 0.2$, $\eta = 0.3$ and $\eta = 0.4$, respectively. We conclude that expansion methods $\textit{PoBA}$ and $\textit{PoVar}$ are both very competitive for the largest tolerance $\tau = 0.01$ for all coefficients $\eta$, in line with our analysis in the main paper with $\eta = 0.1$. In particular, it outperforms \textit{iterative}, and direct factorization (dashed green curves) shows very poor performance due to its lack of scalability. For highest accuracy $\tau = 0.003$ and $\tau = 0.001$, \textit{PoVar} clearly outperforms all its competitors, in line with the main paper. Note that for each $\eta$, we have randomly selected a new set of 97 problems. We can also conclude from our ablation study that our analysis is robust to random initialization.

On the other hand, we link the speed-up of \textit{PoVar} and \textit{RiPoBA}, to the size of the problems. \Cref{fig:small_medium_scale} represents the performance profiles for solving the first two stages by considering only the BAL datasets with less than $1000$ poses (small to medium-scale problems), and \Cref{fig:large_scale} the performance profiles by considering only the BAL datasets with more than $1000$ poses (large-scale problems).

\begin{figure}[tb]
\begin{center}
\includegraphics[width=1\textwidth]{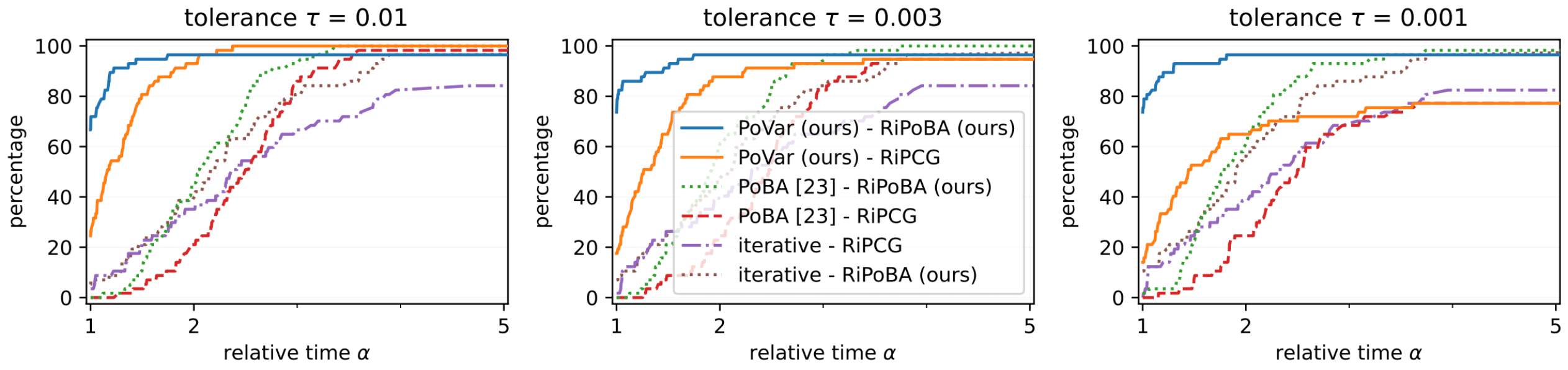}
\caption{Performance profiles for real-world BAL problems with less than $1000$ poses for solving the first two stages Eq. (6) and Eq. (7).}
\label{fig:small_medium_scale}
\end{center}
\end{figure}

\begin{figure}[tb]
\begin{center}
\includegraphics[width=1\textwidth]{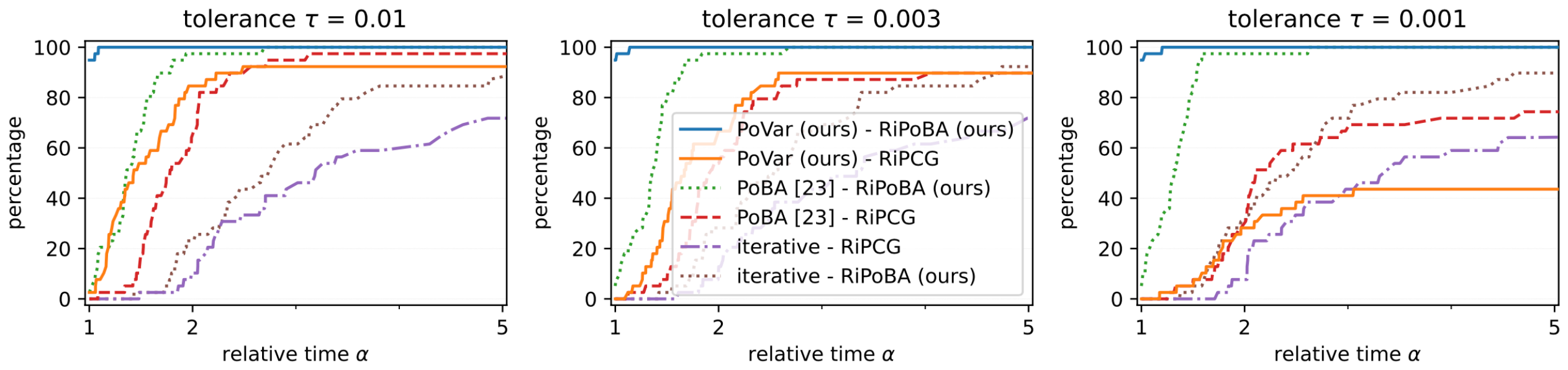}
\caption{Performance profiles for real-world BAL problems with more than $1000$ poses for solving the first two stages Eq. (6) and Eq. (7).}
\label{fig:large_scale}
\end{center}
\end{figure}

\section{Proof of Lemma 1}\label{sec:lemma}
The proof in Weber et al. [23] uses the positive-definiteness of $U_{\lambda}$ and $S$ -- that still holds, to show that $\mu < 1$. It uses the positive-semi-definiteness of $U_{\lambda}^{-\frac{1}{2}}WV_{\lambda}^{-1}W^{\top}U_{\lambda}^{-\frac{1}{2}}$ to conclude that $\mu \geq 0$. For our VarPro formulation, we consider $V_{0}$ instead of $V_{\lambda}$. Note that the generalized Schur complement is written as: $S = U_{\lambda} - W V_{0}^{\dag}W^{\top}$, where $V_{0}^{\dag}$ is the (Moore-Penrose) pseudo-inverse of $V_{0}$. Nevertheless, it is straightforward that $U_{\lambda}^{-\frac{1}{2}}WV_{0}^{\dag}W^{\top}U_{\lambda}^{-\frac{1}{2}}$ is also symmetric positive semi-definite, and then the proof stays almost the same. In details, here is the adapted proof:
\begin{proof}
On the one hand $U_{\lambda}^{-\frac{1}{2}}WV_{0}^{\dag}W^{\top}U_{\lambda}^{-\frac{1}{2}}$ is symmetric positive semi-definite, as $U_{\lambda}$ is symmetric positive definite, and $V_{0}$ is symmetric positive semi-definite. Then its eigenvalues are greater than $0$. As $U_{\lambda}^{-\frac{1}{2}}WV_{0}^{\dag}W^{\top}U_{\lambda}^{-\frac{1}{2}}$ and $U_{\lambda}^{-1}WV_{0}^{\dag}W^{\top}$ are similar, 
\begin{equation}
    \mu \geq 0 \, .
\end{equation} On the other hand $U_{\lambda}^{-\frac{1}{2}}SU_{\lambda}^{-\frac{1}{2}}$ is symmetric positive definite as $S$ and $U_{\lambda}$ are. It follows that the eigenvalues of $U_{\lambda}^{-1}S$ are all strictly positive due to its similarity with $U_{\lambda}^{-\frac{1}{2}}SU_{\lambda}^{-\frac{1}{2}}$. As 
\begin{equation}
    U_{\lambda}^{-1}WV_{0}^{\dag}W^{\top} = I - U_{\lambda}^{-1}S \, ,
\end{equation}
it follows that 
\begin{equation}
    \mu < 1 \, ,
\end{equation} 
that concludes the proof. 
\end{proof}
Concerning Riemannian manifold optimization framework, as the projection $x^{\perp}$ is full rank, it follows that $\tilde{U}_{\tilde{\lambda}}$ and $\tilde{V}_{\tilde{\lambda}}$ are symmetric positive-definite. Then, the previous proof can be very easily adapted to prove Lemma 2.

\section{pOSE Formulation}\label{sec:pose_formulation}
We extensively use the pOSE formulation [11] for testing our solvers. Recently, the follow-up expOSE formulation [14] has been proposed to override some limitations of pOSE. However, such formulation raises some issues for the scalability analysis. In addition to the fact that the authors wrongly claim that they use VarPro, expOSE requires an experimental preprocessing step over each image measurements. Without this first step, the exponential function is equal to $0$ and the algorithm does not update. Nevertheless, such preprocessing is not feasible, in terms of runtime, when the considered dataset is large enough -- which is the topic of our paper, where the number of observations goes up to several tens of millions. Although interesting, expOSE is so far limited to small-scale problems, in line with the problems used by the authors -- between $19$ and $30$ poses in their core paper. Extending this pseudo object space error to large-scale formulation is an interesting research direction, orthogonal to our work.

That being said, note that our proposed solver $\textit{PoVar}$ can be used for solving generic nonlinear problems, and is not restricted to the pOSE formulation. In particular, a recent formulation RpOSE \citelatex{iglesias2021radial} extends pOSE to take into account unknown intrinsics and can be easily adapted to \textit{PoVar}.

\section{Metric Upgrade}\label{sec:metric}
The third stage of pOSE [11] is the \textit{autocalibration} step (see e.g. \citelatex{pollefeys1999self}), aiming to find an ambiguity matrix $H \in \mathbb{R}^{4 \times 4}$ that forces the camera matrices to satisfy the $SE(3)$ constraints, that is to find $H$ such that, for all poses $i$:
\begin{equation}
    x_{p}^{i}H = x_{p}^{i} \begin{pmatrix} A && 0 \\ c^{\top} && 1 \end{pmatrix} \approx K_{i} [R_{i} t_{i}] \, ,
\end{equation}
where $ \begin{pmatrix} c^{\top} && 1 \end{pmatrix}$ represents the plane at infinity.
By denoting $\tilde{H}$ the three left-most columns of $H$, the $SE(3)$ constraint leads to 
\begin{equation}
(K_{i}^{-1}x_{p}^{i}) \tilde{H} \tilde{H}^{\top} (K_{i}^{-1}x_{p}^{i})^{\top} \approx I.
\end{equation}
We find $c$ and the camera scales $\alpha_{i}$ by solving:
\begin{equation}
    \min_{c, \{\alpha_{i}\}} \sum_{i=1}^{n_{p}} \lVert \alpha_{i} (K_{i}^{-1}x_{p}^{i})\tilde{H}(c) \tilde{H}(c)^{\top} (K_{i}^{-1}x_{p}^{i})^{\top} - I \rVert_{F}^{2} \, ,
\end{equation}
with the VarPro algorithm.

In particular, we use the chain rule and the following theorem \citelatex{wang2013derivatives}:
\begin{theorem}
The derivative of $\tilde{H} \tilde{H}^{\top}$ with respect to $\tilde{H}$ is equal to:
\begin{equation}
    \frac{d \tilde{H} \tilde{H}^{\top}}{d \tilde{H}} = (I \otimes \tilde{H}^{\top}) + (\tilde{H}^{\top} \otimes I) T \, ,
\end{equation}
where $T$ is the matrix that transforms $vec(\tilde{H})$ in $vec(\tilde{H}^{\top})$:
\begin{equation}
T vec(\tilde{H}) = vec(\tilde{H}^{\top}) \, ,
\end{equation}
and $vec(H)$ is the operator that creates vector by stringing together the columns of $H$.
\end{theorem}

\section{Dataset}\label{sec:dataset}{

\makeatletter
\newcommand*\ExpandableInput[1]{\input{#1} }
\makeatother
\setlength{\tabcolsep}{1em}
\setlength{\LTcapwidth}{0.99\textwidth}
\begin{longtable}{l r r r}%
\label{tab:problem-size}
\endfirsthead
\endhead
\toprule
\input{problem_size.tex}

\bottomrule
\caption{List of all 97 BAL problems [3] including number of cameras, landmarks and observations.}
\end{longtable}
}
\bibliographystylelatex{splncs04}

\end{document}

%% file: problem_size.tex
&\multicolumn{1}{c}{cameras}&\multicolumn{1}{c}{landmarks}&\multicolumn{1}{c}{observations} \\
\midrule
ladybug-49&49&7,766&31,812\\%
ladybug-73&73&11,022&46,091\\%
ladybug-138&138&19,867&85,184\\%
ladybug-318&318&41,616&179,883\\%
ladybug-372&372&47,410&204,434\\%
ladybug-412&412&52,202&224,205\\%
ladybug-460&460&56,799&241,842\\%
ladybug-539&539&65,208&277,238\\%
ladybug-598&598&69,193&304,108\\%
ladybug-646&646&73,541&327,199\\%
ladybug-707&707&78,410&349,753\\%
ladybug-783&783&84,384&376,835\\%
ladybug-810&810&88,754&393,557\\%
ladybug-856&856&93,284&415,551\\%
ladybug-885&885&97,410&434,681\\%
ladybug-931&931&102,633&457,231\\%
ladybug-969&969&105,759&474,396\\%
ladybug-1064&1,064&113,589&509,982\\%
ladybug-1118&1,118&118,316&528,693\\%
ladybug-1152&1,152&122,200&545,584\\%
ladybug-1197&1,197&126,257&563,496\\%
ladybug-1235&1,235&129,562&576,045\\%
ladybug-1266&1,266&132,521&587,701\\%
ladybug-1340&1,340&137,003&612,344\\%
ladybug-1469&1,469&145,116&641,383\\%
ladybug-1514&1,514&147,235&651,217\\%
ladybug-1587&1,587&150,760&663,019\\%
ladybug-1642&1,642&153,735&670,999\\%
ladybug-1695&1,695&155,621&676,317\\%
ladybug-1723&1,723&156,410&678,421\\%
\midrule
&\multicolumn{1}{c}{cameras}&\multicolumn{1}{c}{landmarks}&\multicolumn{1}{c}{observations} \\
\midrule
trafalgar-21&21&11,315&36,455\\%
trafalgar-39&39&18,060&63,551\\%
trafalgar-50&50&20,431&73,967\\%
trafalgar-126&126&40,037&148,117\\%
trafalgar-138&138&44,033&165,688\\%
trafalgar-161&161&48,126&181,861\\%
trafalgar-170&170&49,267&185,604\\%
trafalgar-174&174&50,489&188,598\\%
trafalgar-193&193&53,101&196,315\\%
trafalgar-201&201&54,427&199,727\\%
trafalgar-206&206&54,562&200,504\\%
trafalgar-215&215&55,910&203,991\\%
trafalgar-225&225&57,665&208,411\\%
trafalgar-257&257&65,131&225,698\\%
\midrule
&\multicolumn{1}{c}{cameras}&\multicolumn{1}{c}{landmarks}&\multicolumn{1}{c}{observations} \\
\midrule
dubrovnik-16&16&22,106&83,718\\%
dubrovnik-88&88&64,298&383,937\\%
dubrovnik-135&135&90,642&552,949\\%
dubrovnik-142&142&93,602&565,223\\%
dubrovnik-150&150&95,821&567,738\\%
dubrovnik-161&161&103,832&591,343\\%
dubrovnik-173&173&111,908&633,894\\%
dubrovnik-182&182&116,770&668,030\\%
dubrovnik-202&202&132,796&750,977\\%
dubrovnik-237&237&154,414&857,656\\%
dubrovnik-253&253&163,691&898,485\\%
dubrovnik-262&262&169,354&919,020\\%
dubrovnik-273&273&176,305&942,302\\%
dubrovnik-287&287&182,023&970,624\\%
dubrovnik-308&308&195,089&1,044,529\\%
dubrovnik-356&356&226,729&1,254,598\\%
\midrule
&\multicolumn{1}{c}{cameras}&\multicolumn{1}{c}{landmarks}&\multicolumn{1}{c}{observations} \\
\midrule
venice-52&52&64,053&347,173\\%
venice-89&89&110,973&562,976\\%
venice-245&245&197,919&1,087,436\\%
venice-427&427&309,567&1,695,237\\%
venice-744&744&542,742&3,054,949\\%
venice-951&951&707,453&3,744,975\\%
venice-1102&1,102&779,640&4,048,424\\%
venice-1158&1,158&802,093&4,126,104\\%
venice-1184&1,184&815,761&4,174,654\\%
venice-1238&1,238&842,712&4,286,111\\%
venice-1288&1,288&865,630&4,378,614\\%
venice-1350&1,350&893,894&4,512,735\\%
venice-1408&1,408&911,407&4,630,139\\%
venice-1425&1,425&916,072&4,652,920\\%
venice-1473&1,473&929,522&4,701,478\\%
venice-1490&1,490&934,449&4,717,420\\%
venice-1521&1,521&938,727&4,734,634\\%
venice-1544&1,544&941,585&4,745,797\\%
venice-1638&1,638&975,980&4,952,422\\%
venice-1666&1,666&983,088&4,982,752\\%
venice-1672&1,672&986,140&4,995,719\\%
venice-1681&1,681&982,593&4,962,448\\%
venice-1682&1,682&982,446&4,960,627\\%
venice-1684&1,684&982,447&4,961,337\\%
venice-1695&1,695&983,867&4,966,552\\%
venice-1696&1,696&983,994&4,966,505\\%
venice-1706&1,706&984,707&4,970,241\\%
venice-1776&1,776&993,087&4,997,468\\%
venice-1778&1,778&993,101&4,997,555\\%
\midrule
&\multicolumn{1}{c}{cameras}&\multicolumn{1}{c}{landmarks}&\multicolumn{1}{c}{observations} \\
\midrule
final-93&93&61,203&287,451\\%
final-394&394&100,368&534,408\\%
final-871&871&527,480&2,785,016\\%
final-961&961&187,103&1,692,975\\%
final-1936&1,936&649,672&5,213,731\\%
final-3068&3,068&310,846&1,653,045\\%
final-4585&4,585&1,324,548&9,124,880\\%
final-13682&13,682&4,455,575&28,973,703\\%